%% file: OFLCA.tex
\documentclass[twoside]{article} 
\usepackage[accepted]{aistats2019}
\usepackage{amsmath}
\usepackage{amsthm}
\usepackage{algorithm, algorithmic}
\usepackage{csquotes}

\usepackage{xcolor}

\usepackage{amssymb}
\usepackage{color}
\usepackage{url}
\usepackage{enumerate}
\usepackage{sidecap}

\usepackage[utf8]{inputenc} % allow utf-8 input
\usepackage[T1]{fontenc}    % use 8-bit T1 fonts
\usepackage{amsfonts}       % blackboard math symbols
\usepackage{nicefrac}       % compact symbols for 1/2, etc.
\usepackage{microtype}      % microtypography

\usepackage{lmodern}  %For compatibility with Shiqiang's Latex compiler

 \newtheorem{theorem}{Theorem}[section]

\newtheorem{lemma}[theorem]{Lemma}

\newtheorem{definition}[theorem]{Definition}

\allowdisplaybreaks

\newcommand{\lo}{[n]}
\newcommand{\spl}[1]{X^{#1}}
\newcommand{\cst}[2]{c^{#1}_{#2}}
\newcommand{\gra}[2]{g^{#1}_{#2}}
\newcommand{\rwd}[2]{r^{#1}_{#2}}
\newcommand{\lt}[1]{\mu^{#1}(\spl{#1})}
\newcommand{\pos}{\mathcal{P}}
\newcommand{\we}[2]{\omega^{#1}_{#2}}
\newcommand{\lr}[1]{\eta^{#1}}

\newcommand{\wv}[1]{\gamma_{#1}}
\newcommand{\fun}[2]{f^{#1}_{#2}}
\newcommand{\tf}[1]{h^{#1}}
\newcommand{\der}[1]{\nabla{#1}}
\newcommand{\fp}{S}
\newcommand{\fpv}[1]{\phi_{#1}}

\newcommand{\id}[1]{\mathcal{I}\left(#1\right)}
\newcommand{\ex}[1]{\mathbb{E}\left(#1\right)}
\newcommand{\pr}[1]{\mathbb{P}\left(#1\right)}

\newcommand{\s}[2]{\sigma(#1,#2)}

\newcommand{\pd}[2]{\partial_{#1}{#2}}

\newcommand{\cc}{C}

\newcommand{\prj}[2]{\Pi_{#1}(#2)}

\newcommand{\nd}{D}
\newcommand{\dd}{d}
\newcommand{\ev}{E}
\newcommand{\pp}[1]{\rho_{#1}}
\newcommand{\pe}[1]{\ev_{#1}}
\newcommand{\pf}[1]{\ev_{#1}}

\newcommand{\nt}[1]{\neg{#1}}

\newcommand{\sus}{Z}

\newcommand{\sqn}[1]{\|#1\|}

\newcommand{\mr}{\hat{r}}
\newcommand{\mc}{\hat{c}}

\newcommand{\plr}[1]{\hat{\eta}^{#1}}

\newcommand{\fc}[1]{z_{#1}}
\newcommand{\ub}{\beta}
\newcommand{\pa}[1]{\Omega_{#1}}

\newcommand{\ca}{\tau}
\newcommand{\cb}{\delta}
\newcommand{\nm}[1]{\zeta^t_{#1}}
\newcommand{\nem}{\Gamma}
\newcommand{\tw}[2]{\we{#1}{#2}}
\newcommand{\al}[2]{\xi^t_{#1,#2}}
\newcommand{\sz}[1]{\pi^t_{#1}}

\newcommand{\gb}{s}
\newcommand{\dis}{\alpha}

\newcommand{\dpr}[2]{\hat{\mu}_{#2}^{#1}}

\newcommand{\rv}[1]{\boldsymbol{r}^{#1}}
\newcommand{\cv}[1]{\boldsymbol{c}^{#1}}
\newcommand{\wev}[1]{\boldsymbol{\omega}^{#1}}
\newcommand{\bv}{\boldsymbol{z}}
\newcommand{\gv}[1]{\boldsymbol{g}^{#1}}
\newcommand{\yv}[1]{\boldsymbol{y}^{#1}}
\newcommand{\gav}{\boldsymbol{\gamma}}
\newcommand{\pv}{\boldsymbol{\phi}}
\newcommand{\bx}{\boldsymbol{x}}

\newcommand{\bpc}[1]{\boldsymbol{c}^{#1+}}
\newcommand{\bnc}[1]{\boldsymbol{c}^{#1-}}
\newcommand{\cpc}[2]{{c}^{#1+}_{#2}}
\newcommand{\cnc}[2]{{c}^{#1-}_{#2}}

\begin{document}

\twocolumn[
\aistatstitle{$\textsc{MaxHedge:}$ Maximising a Maximum Online}
\aistatsauthor{ Stephen Pasteris \And Fabio Vitale \And Kevin Chan}
\aistatsaddress{ University College London \\ London, UK \\ s.pasteris@cs.ucl.ac.uk\\ \And Sapienza University\\ Italy \& INRIA Lille, France \\ fabio.vitale@inria.fr \And Army Research Lab\\  Adelphi, MD, USA \\ kevin.s.chan.civ@mail.mil}
\aistatsauthor{Shiqiang Wang \And Mark Herbster}
\aistatsaddress{IBM Research\\ Yorktown Heights, NY, USA\\ wangshiq@us.ibm.com\And University College London \\ London, UK \\ m.herbster@cs.ucl.ac.uk} 
\runningauthor{Stephen Pasteris,  Fabio Vitale,  Kevin Chan,  Shiqiang Wang,  Mark Herbster}
]

\input{abstract}
\input{intro}

\input{preliminaries}
\input{problems}
\input{algorithms}
\input{analysis}

\input{specialcases}

\input{conclusions}
\newpage
\input{biblio}
\input{Appendix}

\end{document}

%% file: abstract.tex
% !TEX root = OFLCA.tex

\begin{abstract} 
We introduce a new online learning framework where, at each trial, the learner is required to select a subset of actions from a given known action set. Each action is associated with an energy value, a reward and a cost. The sum of the energies of the actions selected cannot exceed a given energy budget. The goal is to maximise the cumulative {\em profit}, where the profit obtained on a single trial is defined as the difference between the {\em maximum reward} among the selected actions and the {\em sum of their costs}. Action energy values and the budget are known and fixed. All rewards and costs associated with each action change over time and are revealed at each trial only after the learner's selection of actions. Our framework encompasses several online learning problems where the environment changes over time; and the solution trades-off between minimising the costs and maximising the {\em maximum} reward of the selected subset of actions, while being constrained to an action energy budget. 
The algorithm that we propose is efficient and general that may be specialised to multiple natural online combinatorial problems.
\end{abstract}

%% file: intro.tex
% !TEX root = OFLCA.tex
\section{Introduction}
%Online Learning for Maximum Reward with Cost and Energy Constraints

In this paper we propose a novel online framework 
%two new variants of the Online Facility Location problem, 
where learning proceeds in a sequence of trials and the goal is to select, at each trial, a {\em subset of actions}
%placement of facilities
maximising a profit while taking into account a certain constraint. More precisely, we are given a finite set of actions enumerated from $1$ to $n$. Each action is associated with three values: (i) a {\em cost}, (ii) an amount of {\em energy} (both of which are required to perform the given action), as well as (iii) a {\em reward}. Both the cost and the reward associated with each action may change over time, are unknown at the beginning of each trial, and their values are all revealed after the learner's selection. The energy associated with each action is instead known by the learner and does not change over time. At each trial, the learner is required to select a subset of actions such that the {\em sum} of their energies does not exceed a {\em fixed energy budget}. The goal of the learner is to maximise the cumulative {\em profit}, where the profit obtained on a single trial is defined as the difference between the {\em maximum reward} among the selected actions and the {\em sum of their costs}. We denote by $T$ the total number of trials.

Our framework is general and flexible in the sense that it encompasses several online learning problems. In the general case, the main challenge lies in the fact that the rewards are not known at the beginning of each trial, and the learner's profit depends only on the {\em maximum} reward among the selected subset of actions, instead of the {\em sum} of all their rewards. In particular, it is worth mentioning three different problems which can be seen as special cases of our online learning framework; these are variants of the Facility Location problem, the $0$-$1$ Knapsack problem, and the Knapsack Median problem.

 When all action energy values are equal to $0$, we obtain an online learning variant of the Facility Location problem (see, e.g., \cite{Cornuejols90}, \cite{Shmoys01}, \cite{Laoutaris07}). 
 The goal of this specific problem may be viewed as 
selecting and opening a subset of facilities at each given trial, to service
a sequence of users which arrive one at a time.
At any given trial $t$, each action's cost may be interpreted as the cost for opening a facility for the $t$-th user. The reward associated with a given facility represents 
%maximum reward associated with the chosen facilities expresses 
%are associated with 
what the user potentially gains when it is opened. More specifically, the rewards can be seen as quantities dependent on the distance between the user and the facilities in some metric space. In this context, it is reasonable to assume that the profit obtained on a single trial depends {\em only on the maximum reward} among the opened facilities 
%and can be viewed as the net reward based on his/her closest opened facility, which 
taking into account the facility costs.  
This represents a natural setting for the Facility Location problem, because in practical %real-world 
scenarios users may arrive sequentially and each arrival requires a connection to an open facility.  Online versions of the Facility Location problem have been studied in several contexts (see, e.g., \cite{Cygan18}, \cite{Meyerson01}). However, as far as we are aware, our work is the first study of this specific online setting for the Facility Location
 problem. Moreover, this dynamic model is natural and interesting within this context\footnote{Similar arguments hold also for motivating the other two special cases mentioned in this section.}, because in practice the location of the next user, which in turn determines the reward associated with each facility, is often unknown to the learner. At each trial, after the user's location is known, the rewards are then revealed, because disclosing the user's location enables to compute all distances between the facilities and the user, which are previously unknown. In fact, this corresponds to assuming all rewards are revealed at the end of the trial. 
 
When, instead, all rewards are equal to $0$ and all costs are negative, the problem can be seen as an online learning variant of the $0$-$1$ Knapsack problem (see, e.g., \cite{Dantzig57}, \cite{Martello77}). In this case, each action corresponds to an item whose weight is equal to the associated energy, the energy budget represents the knapsack capacity, and the absolute value of each action cost can be viewed as the corresponding item value. Our formulation makes the problem challenging especially because the item values are revealed only after the learner's selection.

Finally, when all costs are equal to $0$, we obtain an online learning variant of the Knapsack Median problem (\cite{Charikar99}, \cite{Kumar12}). In this problem (which is a generalization of the k-median problem -- see, e.g., \cite{Charikar99b}, \cite{Jain01}), we are given a set of clients and facilities, as well as a distance metric. The goal is to open a subset of facilities such that the total connection cost (distance to nearest open facility) of all clients is minimized, while the sum of the open facility weights is limited by a budget threshold. In our framework, at each trial the rewards can express the closeness of each facility, and the action energy budget represents therefore the threshold of sum of the opened facility weights.

\smallskip

This problem has practical applications in multiple domains. In computer networks, network administrators may want to understand where to place network monitors or intrusion detection systems. Network packets or malicious attacks are related to the events playing a crucial role in this scenario, and a limited amount of network resources are available to detect or observe network behavior. Another class of application examples include municipal emergency services. A service center needs to deploy responders (police, paramedics, fire rescue), and with limited resources, personnel must be deployed sparingly.
A further application is related to deploying program instances in a distributed computing environment (e.g. distributed cloud). These systems must respond to user requests and are subject to constraints. Both costs and rewards may be unknown in practical real-world scenarios at the beginning of each trial.

\smallskip

For our general problem we propose and rigorously analyse a very scalable algorithm called $\textsc{MaxHedge}$ based on the complex interplay between satisfying the energy budget constraint and bounding the profit by a concave function, which in turn is related to the online gradient descent algorithm. Moreover, the total time required per trial by our learning strategy is quasi-linear in $n$.
We measure the performance of proposed solution with respect to the difference between its cumulative profit and a discounted cumulative profit of the best fixed subset of actions.%, i.e. focusing on the classical notion of regret for online learning problems, in terms of the number of actions $n$ and the number of trials $T$.

In summary, our framework captures several real-world problems, where the environment changes
over time and the solutions trade-off between minimising the costs and maximising the {\em maximum} reward among the selected subset of actions, while being constrained to an action energy budget.
The framework is very general and the proposed algorithm is very efficient and may be specialised to several natural online combinatorial problems.  Finally we provide a guarantee on the rewards achieved and costs incurred as compared to the best fixed subset of actions.

\subsection{Related Work}
%We compare our algorithm to Hedge, Component Hedge-style algorithms, and Approximation algorithm-based algorithms.
The closest work to our online learning framework is perhaps addressed in \cite{Koolen10}, where the authors describe an online learning algorithm for structured concepts that are formed by components. Each component and each concept can be respectively seen as an action and a feasible subset of actions. 
Despite several similarities, the algorithm they proposed cannot be used when the rewards are non-zero, because we focus on the {\em maximum} reward among the selected subsets of actions, whereas in \cite{Koolen10} the profit corresponds to the {\em sum} of the rewards of all selected components/actions. When all rewards are instead equal to $0$, then we have the online variant of the $0$-$1$ Knapsack problem described above as one of the three special cases of our general problem. In Appendix \ref{knapsack} we prove that if the algorithm presented in \cite{Koolen10} could handle the Knapsack problem, then the classical version of the Knapsack problem could be solved in polynomial time, which therefore implies that it cannot address this problem unless $P=NP$. 

The Hedge Algorithm described in \cite{Freund97} obtains a regret linear in $n$. 
%It compares to the cumulative profit of the best fixed action subset selection instead of the cumulative discounted profit (so is better in this sense). 
However, as we have exponentially many possible subsets of selected actions, a vanilla application of Hedge would require an exponential amount of time and space to solve our problem.

Another class of problems and algorithms that are not far from ours is represented by online decision problems where efficient strategies use, as a subroutine, an approximation algorithm for choosing a concept to maximise an inner product (\cite{Fujita13}, \cite{Kakade07}, \cite{Kalai05}). Again, these learning strategies cannot handle the case of non-zero reward. They also have a significantly higher time complexity than $\textsc{MaxHedge}$ in the case of zero reward.

%Another class of problems and algorithms that are not far from ours is represented by online decision problems where efficient strategies use exponential weighting schemes (see, e.g., \cite{Kalai05} and references therein), which are rediscovered in many areas. Again, these learning strategies cannot handle the objective of maximizing the maximum reward of a selected subset of actions. Moreover, even when all action rewards are equal to $0$, the method presented in \cite{Kalai05} cannot provide a feasible solution unless the Knapsack problem can be solved in polynomial time ({\bf FV}: Proof?). 

%However, the paper entitled ?playing games with approximation algorithms? can do knapsack and has an (arbitrarily - depending on runtime) good approximation factor. However, this algorithm is probably slower than ours - will find out.
%The special case where all rewards are null can be handled by the learning strategy described in \cite{Kakade07}. However, 
%the computational complexity guarantees provided entails that in this case the proposed algorithm requires a worst-case time per trial equal to $\Omega(T)$, whereas the time complexity per trial of the method we propose in this paper is quasi-linear in $n$ (we assume $T\gg n$).

In \cite{Chen18} , \cite{Golovin14}, and \cite{Streeter08}, the authors address the problem of online maximisation of non-negative monotone submodular functions. Although our profit is submodular, it is not necessarily either non-negative or monotone. However, as we shall show in Appendix \ref{sm}, we could, for the facility location special case, combine much of the mechanics of our paper with the second algorithm of \cite{Chen18}, essentially doing gradient ascent with the exact expected profit instead of an approximate expected profit (as is done in $\textsc{MaxHedge}$). Even though this new algorithm could be as efficient as the one presented in our paper, its theoretical guarantees are worse. Note that in the Knapsack and Knapsack Median special cases, the profit is indeed monotone and non-negative. For these special cases, the approach presented in \cite{Streeter08} comes close to solving the problem, but only guarantees that the expectation of the total energy does not exceed the budget rather than the actual total energy. In addition, for the Knapsack Median problem, \cite{Streeter08}  uses quadratic time and space. As far as we are aware, there does not exist any trivial reduction to use \cite{Chen18} or \cite{Golovin14} for these special cases.

%In a very interesting and recent paper \cite{Chen18}, Chen et al. presented two algorithms which we found out that can be used to aid solve our facility location special case, although requiring much of the mechanics of our paper to do so. However, the computational complexity (for the first algorithm) and bounds of these algorithms are significantly worse than those of the one we present here. We give a full comparison in appendix \ref{sm}

%In a very interesting and recent paper \cite{Chen18}, Chen et al. presented two algorithms which we found out that can be applied to solve our problems using some non-trivial reductions. However, it can be proved (see Appendix {\bf 1-2-3 blabla} for further details) that when their algorithmic performance is similar to the one achieved our method, the regret is never better than ours and, most importantly, the time required per trial is at least $\Omega\bigl(\sqrt{T}\bigr)$-many times larger than our time complexity. 

Finally, in \cite{Hazan12} the authors address a problem of online minimisation of submodular function. Our paper maximises, instead of minimising, a submodular function (the profit). This is a very different problem.

%Finally, in \cite{Hazan12} the authors address a problem similar to ours. However, in their formulation problem the goal is to minimize the minimum (maximum???) reward instead of maximizing it, which represents an remarkable difference as the reward obtained in our problem is a sub modular function. ({\bf FV}: please double check).

Some of the techniques behind the development of $\textsc{MaxHedge}$ were inspired by \cite{ServicePlacement}.

%% file: preliminaries.tex
\section{Preliminaries and Problem Setup}

\subsection{Preliminaries}
Vectors in $\mathbb{R}^n$ will be indicated in bold. Given a vector $\boldsymbol{v}\in\mathbb{R}^n$  we define $v_i$ to be its $i$-th component. Given vectors $\boldsymbol{v}, \boldsymbol{x}\in\mathbb{R}^n$ define $\langle\boldsymbol{v}, \boldsymbol{x}\rangle:=\sum_{i=1}^nv_ix_i$. Define $\pos$ to be the set of $n$-dimensional real vectors in which every component is non-negative. Given a closed convex set $\cc\subseteq\mathbb{R}^n$ and a vector $\boldsymbol{x}\in\mathbb{R}^n$, we define $\prj{\cc}{\boldsymbol{x}}$ as the projection (under the Euclidean norm) of $\boldsymbol{x}$ onto $\cc$. Let $\mathbb{N}$ be the set of the positive integers. 
%natural numbers (excluding $0$). 
For $l\in\mathbb{N}$ define $[l]=\{1, 2, 3, ..., l\}$. Given an event $E$ let $\nt E$ be the event that $E$ does {\em not} occur. Let $\pr{E}$ be the probability that event $E$ occurs. For a random variable $Y$, let $\ex{Y}$ be the expected value of $Y$. Given a differentiable function $h:\mathbb{R}^n\rightarrow\mathbb{R}$ and a vector $\boldsymbol{x}\in\mathbb{R}^n$ let $\nabla h(\boldsymbol{x})$ be the derivative of $h$ evaluated at $\boldsymbol{x}$. Let $\partial_i h(\boldsymbol{x})$ be the $i$-th component of $\nabla h(\boldsymbol{x})$.

%% file: problems.tex
\subsection{Problem Setup}\label{s:problemsetup}
In this section we formally define our problem. We have a set of $n$ actions enumerated from $1$ to $n$. Each action $i$ has an \textit{energy} $\fc{i}\in[0,\ub]$ for some $\ub<1$, and on each trial $t$ each action $i$ has a \textit{cost} $\cst{t}{i}\in\mathbb{R}$ (which can be negative) and a \textit{reward} $\rwd{t}{i}\in\mathbb{R}^+$. The learner knows $\bv$, but $\cv{t}$ and $\rv{t}$ are revealed to the learner only at the end of trial $t$. On each trial $t$ the learner has to select a set $\spl{t}\subseteq[n]$ of actions, such that the total energy $\sum_{i\in\spl{t}}\fc{i}$ of the selected actions is no greater than $1$. In selecting the set $\spl{t}$, the learner pays a cost equal to $\sum_{i\in\spl{t}}\cst{t}{i}$. On each trial $t$ the learner then receives the maximum reward $\max_{i\in\spl{t}}\rwd{t}{i}$, over all actions selected (defined as equal to zero if $\spl{t}$ is empty). Hence, the \textit{profit} obtained by the learner on trial $t$ is equal to $\max_{i\in\spl{t}}\rwd{t}{i}-\sum_{i\in\spl{t}}\cst{t}{i}$.
\newline
\newline
Formally, this online problem can be defined as follows: We have a vector $\bv\in\pos$ known to the learner. On trial $t$:
\begin{enumerate}
\item Nature selects vectors $\cv{t}\in\mathbb{R}^n$ and $\rv{t}\in\pos$ (but does not reveal these vectors to learner)
\item Learner selects a set $\spl{t}\subseteq[n]$ with $\sum_{i\in\spl{t}}\fc{i}\leq1$
\item Learner obtains profit:
$$\lt{t}:=\max_{i\in\spl{t}}\rwd{t}{i}-\sum_{i\in\spl{t}}\cst{t}{i}$$
\item $\cv{t}$ and $\rv{t}$ are revealed to the learner.
\end{enumerate}

In this paper we write, for a trial $t$, the cost vector $\cv{t}$ as the sum $\bpc{t}+\bnc{t}$ where $\cpc{t}{i}:=\max\{0,\cst{t}{i}\}$ and $\cnc{t}{i}:=\min\{0,\cst{t}{i}\}$, i.e. $\bpc{t}$ and $\bnc{t}$ are the positive and negative parts of the cost vector, respectively.

In order to bound the cumulative profit of our algorithm we, for some $\dis,\cb\in[0,1]$, some set $\fp\subseteq[n]$ and some trial $t$, define the $(\dis,\cb)$-{\em discounted} profit $\dpr{t}{\dis,\cb}(\fp)$ as:

$$\dpr{t}{\dis,\cb}(\fp):=\dis\max_{i\in\fp}\rwd{t}{i}-\dis\sum_{i\in\fp}\cnc{t}{i}-\cb\sum_{i\in\fp}\cpc{t}{i}$$

which would be the profit obtained on trial $t$ if we selected the subset of actions $\fp$, and all the rewards and negative costs were multiplied by $\dis$ and all positive costs were multiplied by $\cb$.

In this paper we provide a randomised quasi-linear time (per trial) algorithm $\textsc{MaxHedge}$ that, for any set $\fp\subseteq[n]$ with $\sum_{i\in\fp}\fc{i}\leq1$, obtains an expected cumulative profit bounded below by:
$$\ex{\sum_{t=1}^T\lt{t}}\geq\sum_{t=1}^T\dpr{t}{\dis,\cb}(\fp)-n\sqrt{2T}\cb(\mr+\mc)$$
where $\cb:=\left(1-\sqrt{\ub}\right)^2$, $\dis:=1-\exp\left(-\left(1-\sqrt{\ub}\right)^2\right)$, $\mr:=\max_{t\in[T], i\in[n]}\rwd{t}{i}$ and $\mc:=\max_{t\in[T], i\in[n]}|\cst{t}{i}|$.

This paper is structured as follows. In Section \ref{algosec} we introduce the algorithms that define $\textsc{MaxHedge}$. In Section \ref{s:feasibility} we prove that the sets of actions selected by $\textsc{MaxHedge}$ are feasible. In Section \ref{s:profit} we prove the above bound on the cumulative profit.  In Section \ref{SCsec} we give special cases of the general problem.

%% file: algorithms.tex
\section{Algorithms}\label{algosec}

We now present our learning strategy $\textsc{MaxHedge}$, describing the two subroutines \enquote{Algorithm 1} and \enquote{Algorithm 2} (see the pseudocode below). $\textsc{MaxHedge}$ maintains a vector $\wev{}\in\cc$ where $\cc:=\{\boldsymbol{x}\in[0,1]^n:\langle\boldsymbol{x},\bv\rangle\leq1\}$. We define $\wev{t}$ to be the vector $\wev{}$ at the start of trial $t$. We initialise $\wev{1}\leftarrow \boldsymbol{0}$. On trial $t$ $\textsc{MaxHedge}$ (randomly) constructs $\spl{t}$ from $\wev{t}$ using Algorithm $1$. After receiving $\rv{t}$ and $\cv{t}$ $\textsc{MaxHedge}$ updates $\wev{}$ (from $\wev{t}$ to $\wev{t+1}$) using Algorithm $2$. Algorithm $2$ also uses a \textit{\enquote{learning rate}} $\plr{t}$ which is defined from $\plr{t-1}$. We define $\plr{0}:=\infty$.

Algorithm $1$ operates with a partition of all possible actions and, for each set in this partition, Algorithm $1$ draws  a certain subset of actions from it. Given a set in the partition, the number of actions drawn from it and the probability distribution governing the draws depends on $\ub$ and $\wev{t}$. The subset, $\spl{t}$, of actions selected by Algorithm $1$ satisfies the following three crucial properties (proved in sections~\ref{s:feasibility} and \ref{s:profit}):
\begin{itemize}
\item The total energy of all actions selected is no greater than $1$.
\item Given an arbitrary set $\sus\subseteq[n]$, the probability that $\spl{t}$ and $\sus$ intersect is lower bounded by $1-\exp\left(-\cb\sum_{i\in \sus}\we{t}{i}\right)$.
\item Given an action $i$, the probability that it is selected on trial $t$ is upper bounded by $\cb\we{t}{i}$.
\end{itemize}
In the analysis we shall construct, from $\rv{t}$ and $\cv{t}$, a convex function $\tf{t}:\cc\rightarrow\mathbb{R}$. Using the second and third properties (given above) of Algorithm $1$, we show that $\tf{t}(\wev{t})$ is an upper bound on the negative of the expected profit on trial $t$. Algorithm $2$ computes the gradient $\gv{t}:=\der{\tf{t}}(\wev{t})$ and updates $\wev{}$ using online gradient descent on $\cc$.

The last line of Algorithm $2$ requires us to project (with Euclidean distance) the vector $\yv{t}$ onto the set $\cc$, i.e. we must compute the $\bx$ that minimises the value $\|\bx-\yv{t}\|$ subject to $\bx\in\cc$. Note that minimising $\|\bx-\yv{t}\|$ is equivalent to minimising $\|\bx-\yv{t}\|^2=\langle\bx,\bx\rangle-2\langle\yv{t},\bx\rangle+\langle\yv{t},\yv{t}\rangle$, which is in turn equivalent to minimising $\langle\bx,\bx\rangle-2\langle\yv{t},\bx\rangle$. The constraints defining the set $\cc$ then imply that this projection is a case of the continuous bounded quadratic knapsack problem which can be solved in linear time (see, e.g., \cite{Kiwiel08}).

The bottleneck of the algorithms is hence the ordering step in Algorithm $2$ which takes a time of $\mathcal{O}(n\log(n))$

\begin{algorithm}\label{BFL}
\caption{Constructing $\spl{t}$}
\begin{algorithmic}[1]
 \STATE $\cc\leftarrow\{\boldsymbol{x}\in[0,1]^n:\langle\boldsymbol{x},\bv\rangle\leq1\}$
 \STATE $\ub\leftarrow\max_{i\in[n]}\fc{i}$
\STATE $\ca\leftarrow1-\sqrt{\ub}$ 
\STATE  $\cb\leftarrow\left(1-\sqrt{\ub}\right)^2$.
\STATE $\nem\leftarrow\{q\in\mathbb{N}:\exists i\in[n] \operatorname{with} \ca^q\ub<\fc{i}\leq\ca^{q-1}\ub \}$
\STATE For all $q\in\nem$ set $\pa{q}\leftarrow\{i\in[n]:\ca^q\ub<\fc{i}\leq\ca^{q-1}\ub\}$. 
\newline
 \STATE {\bf Input:} $\wev{t}\in\cc$
 \newline
\STATE For all $q\in\nem$ set $\sz{q}\leftarrow\sum_{i\in\pa{q}}\tw{t}{i}$
\STATE For all $q\in\nem$ set $\nm{q}\leftarrow\left\lfloor\cb\sz{q}\right\rfloor$
\STATE For all $q\in\nem$ and for all $k\leq\nm{q}$ draw $\al{q}{k}$ randomly from $\pa{q}$ such that $\al{q}{k}\leftarrow i$ with probability $\tw{t}{i}/\sz{q}$
\STATE For all $q\in\nem$ and for $k:=\nm{q}+1$ draw $\al{q}{k}$ randomly from $\pa{q}\cup\{0\}$ such that, for $i\in\pa{q}$ we have $\al{q}{k}\leftarrow i$ with probability $(\cb\sz{q}-\lfloor\cb\sz{q}\rfloor)\tw{t}{i}/\sz{q}$, and we have $\al{q}{k}\leftarrow0$ with probability $\lfloor\cb\sz{q}\rfloor+1-\cb\sz{q}$. NB: In the case that $\sz{q}=0$ we define $0/0=0$
\newline
\STATE {\bf Output:} $\spl{t}\leftarrow\{\al{q}{k}:q\in\nem,k\leq\nm{q}+1\}\setminus\{0\}$
\end{algorithmic}
\end{algorithm}

\begin{algorithm}\label{update}
\caption{Computing $\wev{t+1}$}
\begin{algorithmic}[1]
\STATE $\cc\leftarrow\{\boldsymbol{x}\in[0,1]^n:\langle\boldsymbol{x},\bv\rangle\leq1\}$
 \STATE $\ub\leftarrow\max_{i\in[n]}\fc{i}$
\STATE $\cb\leftarrow\left(1-\sqrt{\ub}\right)^2$
\newline
 \STATE {\bf Input:} $\wev{t}\in\cc$, $\cv{t}\in\mathbb{R}^n$ $\rv{t}\in\pos$
 \newline
\STATE Order $\lo$ as $\lo=\{\s{t}{1}, \s{t}{2}, ...\s{t}{n} \}$ such that $\rwd{t}{\s{t}{j}}\geq\rwd{t}{\s{t}{j+1}}$ for all $j\in[n-1]$
\STATE For all $j\in[n]$ set $\epsilon^t_j\leftarrow\exp\left(-\cb\sum_{k=1}^j\we{t}{\s{t}{k}}\right)$
\STATE For all $j\in[n]$ set:\\ $\lambda^t_j\leftarrow\rwd{t}{\s{t}{n}}\epsilon^t_n+\sum_{k=j}^{n-1}\left(\rwd{t}{\s{t}{k}}-\rwd{t}{\s{t}{k+1}}\right)\epsilon^t_k$ 
\STATE For all $j\in[n]$ set:\\ $\gra{t}{\s{t}{j}}\leftarrow\cb\left(\cpc{t}{\s{t}{j}}+\cnc{t}{\s{t}{j}}\exp\left(-\cb\we{t}{\s{t}{j}}\right)-\lambda^t_j\right)$
\STATE  $\plr{t}\leftarrow \min\left\{\plr{t-1},\sqrt{n}/\sqn{\gv{t}}\right\}$
\STATE  $\lr{t}\leftarrow\plr{t}/\sqrt{2t}$
\STATE  $\yv{t}\leftarrow \wev{t}-\lr{t}\gv{t}$
\newline
\STATE  {\bf Output:} $\wev{t+1}\leftarrow\prj{\cc}{\yv{t}}$
\end{algorithmic}
\end{algorithm}

%% file: analysis.tex
% !TEX root = OFLCA.tex

\section{The Feasibility of $\spl{t}$}\label{s:feasibility}

In this section we show that the the total energy of the actions selected by Algorithm $1$ is no greater than $1$, as required in our problem definition (see Section~\ref{s:problemsetup}). We first introduce the sets and quantities used in the selection of the actions.

\begin{definition}
We define the following:
\begin{itemize}
\item $\ca:=1-\sqrt{\ub}$ and $\cb:=\left(1-\sqrt{\ub}\right)^2$
\item For all $q\in\mathbb{N}$ we define $\pa{q}:=\{i\in[n]:\ca^q\ub<\fc{i}\leq\ca^{q-1}\ub\}$
\item $\nem:=\{q\in\mathbb{N}:\pa{q}\neq\emptyset\}$. Note that $\{\pa{q}:q\in\nem\}$ is a partition of $[n]$
\item On trial $t$, for all $q\in\nem$ define $\sz{q}:=\sum_{i\in\pa{q}}\tw{t}{i}$ and $\nm{q}:=\left\lfloor\cb\sz{q}\right\rfloor$. %$\nm{q}$ is the number of actions in $\pa{q}$ that will be selected on trial $t$ (although we may select the same action twice)
\end{itemize}
\end{definition}
Algorithm $1$ works by drawing actions $\{\al{q}{k}:q\in\nem, k\in[\nm{q}+1]\}$ randomly, where the number of actions (including a ``null action'' $0$) drawn (with replacement) from $\pa{q}$ is equal to $\nm{q}+1$ and the probability distribution of the draws is dependent on $\wev{t}$.
%Algorithm $1$ works as follows: for all $q\in\nem$ and for all $k\leq\nm{q}$ an action $\al{q}{k}$ is randomly drawn from $\pa{q}$ such that $\al{q}{k}= i$ with probability $\tw{t}{i}/\sz{q}$. For $k:=\nm{q}+1$ an action $\al{q}{k}$ is  The learner then chooses $\spl{t}=\{\al{q}{k}:q\in\nem,k\leq\nm{q}\}$.
 \newline
 \newline
 The following theorem ensures that the choice of $\spl{t}$ made by our method satisfies our problem's energy constraint.

\begin{theorem}\label{Feas}
On trial $t$ we have $\sum_{i\in\spl{t}}\fc{i}\leq 1$.
\end{theorem}

\begin{proof}
See Appendix \ref{MP}
\end{proof}

\section{Bounding the Cumulative Profit}\label{s:profit}

In this section we bound the cumulative profit of $\textsc{MaxHedge}$.

\subsection{The Probability of Intersection}
In this subsection we first bound below the probability that, on a trial $t$, an arbitrary set $\sus$ intersects with $\spl{t}$. We start with the following lemma:

\begin{lemma}\label{pl}
Given a set $\nd$ of independent draws from $|\nd|$-many probability distributions, such that the probability of an event $\ev$ happening on draw $\dd\in\nd$ is $\pp{\dd}$, then the probability of event $\ev$ happening on either of the draws is lower bounded by:
$$1-\exp\left(-\sum_{\dd\in\nd}\pp{\dd}\right)$$
\end{lemma}

\begin{proof}
See Appendix \ref{MP}
\end{proof}

We now bound the probability of intersection:

\begin{theorem}\label{pal1}
For any trial $t$ and any subset $\sus\subseteq[n]$ we have $\pr{\sus\cap\spl{t}\neq\emptyset}\geq 1-\exp\left(-\cb\sum_{i\in \sus}\we{t}{i}\right)$.
\end{theorem}

\begin{proof}
See Appendix \ref{MP}
\end{proof}

We now bound the probability that some arbitrary action is selected on trial $t$:
\begin{theorem}\label{ps}
Given some action $i\in[n]$ and some trial $t\in[T]$ we have $1-\exp(-\cb\we{t}{i})\leq\pr{i\in\spl{t}}\leq\cb\we{t}{i}$.
\end{theorem}

\begin{proof}
See Appendix \ref{MP}
\end{proof}

\subsection{Approximating the Expected Profit}
In this subsection we define a convex function $\tf{t}:\cc\rightarrow\mathbb{R}$ and show that the expected profit on trial $t$ is bounded below by $-\tf{t}(\wev{t})$.

\begin{definition}
For each trial $t$ we order $\lo$ as $\lo=\{\s{t}{1}, \s{t}{2}, \ldots, \s{t}{n} \}$ where $\rwd{t}{\s{t}{j}}\geq\rwd{t}{\s{t}{j+1}}$ for all $j\in[n-1]$. We also define $\s{t}{n+1}:=0$ and $\rwd{t}{0}:=0$.
\end{definition}

\begin{definition}
Given a trial $t$ and a number $j\in[n]$ we define the function $\fun{t}{j}:\pos\rightarrow\mathbb{R}$ by: $$\fun{t}{j}(\gav)\! := \!\left(\rwd{t}{\s{t}{j}}-\rwd{t}{\s{t}{j+1}}\right)\!\left(\!1-\exp\!\left(\!-\cb\sum_{k=1}^j\wv{\s{t}{k}}\right)\!\!\right)$$
We also define the function $\tf{t}:\pos\rightarrow\mathbb{R}$ as:
$$\tf{t}(\gav)=\langle\bpc{t},\gav\rangle+\sum_{i=1}^n\cnc{t}{i}(1-\exp(-\cb\wv{i}))-\sum_{j=1}^n\fun{t}{j}(\gav)$$
\end{definition}

\begin{theorem}\label{t1}
For all $t\in[T]$, the function $\tf{t}$ is convex.
\end{theorem}

\begin{proof}
See Appendix \ref{MP}
\end{proof}

The rest of this subsection proves that the expected profit on trial $t$ is bounded below by $-\tf{t}(\wev{t})$.

\begin{lemma}\label{l1}
On trial $t$ we have $\max_{i\in\spl{t}}\rwd{t}{i}=\sum_{j=1}^n\left(\rwd{t}{\s{t}{j}}-\rwd{t}{\s{t}{j+1}}\right)\id{\exists k\leq j: \s{t}{k}\in\spl{t}}$.
\end{lemma}
\begin{proof}
See Appendix \ref{MP}
\end{proof}

\begin{lemma}\label{l2}
On trial $t$ we have $\ex{\max_{i\in\spl{t}}\rwd{t}{i}}=\sum_{j=1}^n\left(\rwd{t}{\s{t}{j}}-\rwd{t}{\s{t}{j+1}}\right)\pr{\exists k\leq j: \s{t}{k}\in\spl{t}}$.
\end{lemma}

\begin{proof}
Direct from lemma \ref{l1} using linearity of expectation.
\end{proof}

\begin{theorem}\label{l4}
On trial $t$ we have $\ex{\lt{t}}\geq-\tf{t}(\wev{t})$.
\end{theorem}

\begin{proof}
See Appendix \ref{MP}
\end{proof}

\subsection{The Gradient}

In this subsection we show how to construct the gradient of $\tf{t}$ and bound its magnitude. We start with the following definitions.

\begin{definition}
On any trial $t$ and for any $j\in[n]$ we define:
\begin{itemize}
\item $\epsilon^t_j:=\exp\left(-\cb\sum_{k=1}^j\we{t}{\s{t}{k}}\right)$
\item $\lambda^t_j:=\sum_{k=j}^{n}\left(\rwd{t}{\s{t}{k}}-\rwd{t}{\s{t}{k+1}}\right)\epsilon^t_k$ 
\item $\gra{t}{\s{t}{j}}:=\cb\left(\cpc{t}{\s{t}{j}}+\cnc{t}{\s{t}{j}}\exp\left(-\cb\we{t}{\s{t}{j}}\right)-\lambda^t_j\right)$
 \end{itemize}
\end{definition}

We first show that $\gv{t}$ is the gradient of $\tf{t}$ evaluated at $\wev{t}$.

\begin{theorem}\label{deriv}
On any trial $t$ we have $\gv{t}=\der{\tf{t}}(\wev{t})$.
\end{theorem}

\begin{proof}
See Appendix \ref{MP}
\end{proof}

We now bound the magnitude of the gradient.

\begin{lemma}\label{grad}
For any trial $t$ we have $\sqn{\gv{t}}^2\leq n\cb^2(\mr+\mc)^2$.
\end{lemma}

\begin{proof}
Since $\we{t}{\s{t}{k}}\geq 0$ for all $k\in[n]$ we have $\epsilon^t_j\in[0,1]$ for all $j\in[n]$. This gives us, for all $j\in[n]$, that $\cb\lambda^t_j\leq\cb\rwd{t}{\s{t}{n}}+\cb\sum_{k=j}^{n-1}(\rwd{t}{\s{t}{k}}-\rwd{t}{\s{t}{k+1}})=\cb\rwd{t}{\s{t}{j}}\leq\cb\mr$. Since also $\lambda^t_j\geq 0$ this implies that $-\gra{t}{i}\leq\cb\mr+\cb\mc$ and that $\gra{t}{i}\leq\cb\mc$ so $(\gra{t}{i})^2\leq(\cb\mr+\cb\mc)^2$. This then implies the result.
\end{proof}

\subsection{Online Gradient Descent}
In this subsection we show that Algorithm $2$ corresponds to the use of online gradient descent over $\cc$ with convex functions $\{\tf{t}:t\in[T]\}$ and we use the standard analysis of online gradient descent to derive a lower bound on the cumulative profit.
\newline
\newline
From here on we compare the performance of our algorithm against any fixed set $S$ of actions such that $\sum_{i\in S}\fc{i}\leq 1$. We define $\pv$ as the vector in $\mathbb{R}^n$ such that, for all $i\in[n]$, we have $\fpv{i}:=0$ if $i\notin S$ and $\fpv{i}:=1$ if $i\in S$. It is clear that $\pv\in\cc$.

\begin{definition}
Our \textit{learning rates} are defined as follows:
\begin{itemize}
\item $\plr{t}:=\min_{t'\leq t}(\sqrt{n}/\|\gv{t'}\|)$
\item $\lr{t}:=\plr{t}/\sqrt{2t}$
\end{itemize}
\end{definition}

The next result follows from the standard analysis of online gradient descent.

\begin{theorem}\label{gdt}
We have:
$$\sum_{t=1}^T(\tf{t}(\wev{t})-\tf{t}(\pv))\leq \frac{R^2}{2\lr{T}}+\frac{1}{2}\sum_{t=1}^T\lr{t}\sqn{\gv{t}}^2$$
where $R:=\max_{x,y\in\cc}\sqn{x-y}$.
\end{theorem}

\begin{proof}
For all trials $t$:
From Theorem \ref{t1} we have that $\tf{t}$ is a convex function. From Theorem \ref{deriv} we have that $\gv{t}=\der{\tf{t}}(\wev{t})$. We also have  that $\lr{t+1}\leq\lr{t}$ so since, by Algorithm 2, we have $\wev{t+1}=\prj{\cc}{\wev{t}-\lr{t}\gv{t}}$ and $\pv\in\cc$, the standard analysis of online gradient descent (see, e.g., \cite{Zinkevich03}) gives us the result.
\end{proof}

We now bound the right hand side of the equation in Theorem \ref{gdt}.

\begin{definition}\label{maxgrad}
Define $\gb:=\max_{t\in[T]}\sqn{\gv{t}}^2/n$.
\end{definition}

\begin{lemma}\label{Gl}
For any trial $t$, $\lr{t}\sqn{\gv{t}}^2\leq n\sqrt{s/(2t)}$.
\end{lemma}

\begin{proof}
We have $\plr{t}=\min_{t'\leq t}(\sqrt{n}/\|\gv{t'}\|)\leq\sqrt{n}/\|\gv{t}\|)$. This implies that that $\lr{t}=\plr{t}/\sqrt{2t}\leq \sqrt{n}/(\|\gv{t}\|\sqrt{2t})$ so $\lr{t}\sqn{\gv{t}}^2\leq\sqrt{n}\|\gv{t}\|/\sqrt{2t}$ which, by definition of $\gb$, is bounded above by $ n\sqrt{s/(2t)}$.
\end{proof}

\begin{lemma}\label{Rl}
${R^2}/{\lr{T}}\leq  n\sqrt{2sT}$
\end{lemma}

\begin{proof}
We have $\plr{T}=\min_{t\in[T]}\sqrt{n}/\sqn{\gv{t}}=1/\sqrt{\gb}$ so $\lr{T}:=\plr{T}/\sqrt{2T}\leq1/\sqrt{2\gb T}$. Since $\cc\subseteq[0,1]^n$, if $\boldsymbol{x},\boldsymbol{y}\in\cc$ then each component of $\boldsymbol{x}-\boldsymbol{y}$ has magnitude bounded above by $1$ which implies that $R^2\leq n$. Putting together gives the result.
\end{proof}

\begin{theorem}\label{bound}
We have:
$$\sum_{t=1}^T(\tf{t}(\wev{t})-\tf{t}(\pv))\leq n\sqrt{2T}\cb(\mr+\mc)$$
\end{theorem}

\begin{proof}
See Appendix \ref{MP}
\end{proof}

The next result bounds $\tf{t}(\pv)$.

\begin{lemma}\label{l5}
On trial $t$ we have $\dpr{t}{\dis,\cb}(\fp)\leq-\tf{t}(\pv)$ where $\dis:=1-e^{-\cb}$.
\end{lemma}

\begin{proof}
Let $j'=\operatorname{argmax}_{j\in[n]:\s{t}{j}\in\fp}\rwd{t}{\s{t}{j}}$ which is equal to the minimum $j$ such that $\s{t}{j}\in\fp$. Note that for all $j\geq j'$ we have $\sum_{k=1}^j\fpv{\s{t}{j}}\geq\,\!1$ and hence $\fun{t}{j}(\pv)\geq\left(\rwd{t}{\s{t}{j}}-\rwd{t}{\s{t}{j+1}}\right)(1-e^{-\cb})$ so $\sum_{j=1}^n\fun{t}{j}(\pv)\geq\sum_{j=j'}^n\fun{t}{j}(\pv)\geq\sum_{j=j'}^n\left(\rwd{t}{\s{t}{j}}-\rwd{t}{\s{t}{j+1}}\right)(1-e^{-\cb})=\rwd{t}{\s{t}{j'}}(1-e^{-\cb})=(1-e^{-\cb})\max_{i\in\fp}\rwd{t}{i}$.

Also note that $\cb\langle\bpc{t},\pv\rangle=\cb\sum_{i\in\fp}\cpc{t}{i}$ and that $\sum_{i=1}^n\cnc{t}{i}(1-\exp(-\cb\fpv{i}))=\sum_{i\in\fp}\cnc{t}{i}(1-e^{-\cb})$. Combining with the above gives us the result.
\end{proof}

%\begin{proof}
%See Appendix \ref{MP}
%\end{proof}

Putting together we obtain the main result:

\begin{theorem}
We have:
$$\ex{\sum_{t=1}^T\lt{t}}\geq \sum_{t=1}^T\dpr{t}{\dis,\cb}(\fp)-n\sqrt{2T}\cb(\mr+\mc)$$ where $\cb:=\left(1-\sqrt{\ub}\right)^2$, $\dis:=1-\exp\left(-\left(1-\sqrt{\ub}\right)^2\right)$, $\mr:=\max_{t\in[T], i\in[n]}\rwd{t}{i}$ and $\mc:=\max_{t\in[T], i\in[n]}|\cst{t}{i}|$.
\end{theorem}

\begin{proof}
Let $\dis:=1-e^{-\cb}$. By Theorem \ref{l4} we have, for all $t\in[T]$, that $\ex{\lt{t}}\geq-\tf{t}(\wev{t})$. By Lemma \ref{l5} we have, for all $t\in[T]$, that $\dpr{t}{\dis,\cb}(\fp)\leq-\tf{t}(\pv)$. Hence we have that $\dpr{t}{\dis,\cb}(\fp)-\ex{\lt{t}}\leq\tf{t}(\wev{t})-\tf{t}(\pv)$. By Theorem \ref{bound} we than have:
\begin{align*}
&\sum_{t=1}^T\left(\dpr{t}{\dis,\cb}(\fp)-\ex{\lt{t}}\right)\\
\leq&\sum_{t=1}^T(\tf{t}(\wev{t})-\tf{t}(\pv))\\
\leq& n\sqrt{2T}\cb(\mr+\mc)
\end{align*}
Rearranging gives us:
$$\sum_{t=1}^T\ex{\lt{t}}\geq\sum_{t=1}^T\dpr{t}{\dis,\cb}(\fp)-n\sqrt{2T}\cb(\mr+\mc)$$
\end{proof}

%% file: specialcases.tex
\section{Special Cases}\label{SCsec}

The following online variants of classic computer science problems are special cases of the general problem.

\subsection{Facility Location Problem}
The (inverted) facility location problem is defined by a vector $\boldsymbol{c}\in\pos$ and vectors $\boldsymbol{r}^1, \boldsymbol{r}^2,\cdots, \boldsymbol{r}^T\in\pos$. A feasible solution is any $\spl{}\subseteq[n]$. The aim is to maximise the objective function:
$$\sum_{t=1}^T\max_{i\in X}\rwd{t}{i}-\sum_{i\in X}c_i$$

An example of the problem is as follows. There are $n$ sites and $T$ users, all located in some metric space. We have to choose a set $X$ of sites to open a facility on. Opening a facility on site $i$ costs us $c_i$. Each user pays us a reward based on how near it is to the closest open facility. If the nearest open facility to user $t$ is at site $i$ then user $t$ rewards us $\rwd{t}{i}$. The objective is to maximise the total profit.

In our online variant of the (inverted) facility location problem, learning proceeds in trials. On trial $t$:
\begin{enumerate}
\item For all sites $i$, the cost, $\cst{t}{i}$ of opening a facility on site $i$ is revealed to the learner.
\item The learner chooses a set $\spl{t}$ of sites in which to open facilities on.
\item User $t$ requests the use of a facility, revealing $\boldsymbol{r}^t$ to the learner.
\item Learner incurs profit: $\max_{i\in\spl{t}}\rwd{t}{i}-\sum_{i\in\spl{t}}\cst{t}{i}$
\end{enumerate}

The objective is to maximise the cumulative profit. Note that this is the special case of our problem when, for all $i\in[n]$ and $t\in[T]$ we have $\fc{i}=0$ and $\cst{t}{i}\geq 0$. Given some set $S$ the expected cumulative profit of \textsc{MaxHedge} is then bounded below by:
$$\sum_{t=1}^T\left((1-1/e)\max_{i\in S}\rwd{t}{i} -\sum_{i\in S}\cst{t}{i}\right)-n\sqrt{2T}(\mr+\mc)$$

\subsection{Knapsack Median Problem}
The (inverted) knapsack median problem is defined by a vector $\boldsymbol{z}\in\pos$ and vectors $\boldsymbol{r}^1, \boldsymbol{r}^2,\cdots, \boldsymbol{r}^T\in\pos$. A feasible solution is any $\spl{}\subseteq[n]$ with $\sum_{i\in\spl{}}\fc{i}\leq 1$. The aim is to maximise the objective function
$\sum_{t=1}^T\max_{i\in X}\rwd{t}{i}$.

An example of the problem is as follows. There are $n$ sites and $T$ users, all located in some metric space. We have to choose a set $X$ of sites to open a facility on. Opening a facility on site $i$ has a fee of $\fc{i}$ and we have a budget of $1$ to spend on opening facilities. Each user pays us a reward based on how near it is to the closest open facility. If the nearest open facility to user $t$ is at site $i$ then user $t$ rewards us $\rwd{t}{i}$. The objective is to maximise the total reward.

In our online variant of the (inverted) knapsack median problem, learning proceeds in trials. The learner has knowledge of the fee $\fc{i}$ for every site $i$. On trial $t$:
\begin{enumerate}
\item The learner chooses a set $\spl{t}$ of sites in which to open facilities on. The total fee, $\sum_{i\in\spl{t}}\fc{i}$, can't exceed $1$.
\item User $t$ requests the use of a facility, revealing $\boldsymbol{r}^t$ to the learner.
\item Learner incurs reward: $\max_{i\in\spl{t}}\rwd{t}{i}$
\end{enumerate}
The objective is to maximise the cumulative reward.
Note that this is the special case of our problem when, for all $i\in[n]$ and $t\in[T]$ we have $\cst{t}{i}=0$. Given some set $S$ with $\sum_{i\in S}\fc{i}\leq 1$ the expected cumulative reward of \textsc{MaxHedge} is then bounded below by:
$$\left(1-\exp\left(-\cb\right)\right)\sum_{t=1}^T\max_{i\in S}\rwd{t}{i}-\cb n\sqrt{2T}\mr$$
where $\cb:=(1-\sqrt{\max_{i\in[n]}\fc{i}})^2$

\subsection{0-1 Knapsack Problem}
The knapsack problem is defined by a vector $\boldsymbol{z}\in \pos$ and a vector $\boldsymbol{v}\in\pos$. A feasible solution is any $\spl{}\subseteq[n]$ with $\sum_{i\in\spl{}}\fc{i}\leq 1$. The aim is to maximise the objective function $\sum_{i\in\spl{}}v_i$.

An example of the problem is as follows. We have $n$ items. Item $i$ has a value $v_i$ and a weight $\fc{i}$. The objective is to place a set $X\subseteq[n]$ of items in the knapsack that maximises the total value of all items in the knapsack subject to their total weight being no greater than $1$.

In our online variant of the knapsack problem, learning proceeds in trials. The learner has knowledge of the weight $\fc{i}$ for every item $i$. On trial $t$:
\begin{enumerate}
\item The learner chooses a set $\spl{t}$ of items to place in the knapsack. The total weight, $\sum_{i\in\spl{t}}\fc{i}$, can't exceed $1$.
\item For each item $i$, the value $v^t_i$, of item $i$ on this trial is revealed to the learner
\item Learner incurs profit: $\sum_{i\in\spl{t}}v^t_i$
\end{enumerate}
The objective is to maximise the cumulative profit.
Note that this is the special case of our problem when, for all $i\in[n]$ and $t\in[T]$ we have $\rwd{t}{i}=0$ and $\cst{t}{i}\leq 0$ (noting that $\cst{t}{i}=-v^t_i$).  Given some set $S$ with $\sum_{i\in S}\fc{i}\leq 1$ the expected cumulative profit of \textsc{MaxHedge} is then bounded below by:
$$\left(1-\exp\left(-\cb\right)\right)\sum_{t=1}^T\sum_{i\in S}v^t_i-\cb n\sqrt{2T}\mc$$
where $\cb:=(1-\sqrt{\max_{i\in[n]}\fc{i}})^2$

%% file: conclusions.tex
% !TEX root = OFLCA.tex
\section{Conclusions and Ongoing Research}
We presented and investigated in depth a novel online framework, capable of encompassing several online learning problems and capturing many practical problems in the real-world.
The main challenge of the general version of this problem lies in the fact that the learner's profit depends on the maximum reward of the selected actions, instead of the sum of all their rewards. 
We proposed and rigorously analysed a very scalable and efficient learning strategy $\textsc{MaxHedge}$. %which is based on the main characteristics of the $\textsc{Hedge}$ algorithm, while still being conceptually far from its standard formulation.

Current ongoing research includes:
\begin{itemize}
\item Deriving a lower bound on the achievable profit.
\item Complementing our results with a set of experiments on synthetic and real-world datasets. 
\item Several real systems usually have a switching cost for turning on/off services, which translates in our framework to the cost incurred whenever an action selected at any given trial is not selected in the preceding one. This represents an interesting direction for further research, which is certainly motivated by practical problems. 
\end{itemize}
%{\bf Acknowledgements.}  This research was sponsored by the U.S. ARL and the U.K. MOD under Agreement Number W911NF-16-3-0001. The views and conclusions contained in this document are those of the authors and should not be interpreted as representing the official policies, either expressed or implied, of the U.S. ARL, the U.S. Government, the U.K. MOD or the U.K. Government. Fabio Vitale acknowledges support from the ERC Starting Grant ` DMAP 680153'', the Google Focused Award ``ALL4AI'', and grant ``Dipartimenti di Eccellenza 2018-2022'', awarded to the Department of Computer Science of Sapienza University.

{\bf Acknowledgements.} 
This research was sponsored by the U.S. Army Research Laboratory and the U.K. Ministry of Defence under Agreement Number W911NF-16-3-0001. The views and conclusions contained in this document are those of the authors and should not be interpreted as representing the official policies, either expressed or implied, of the U.S. Army Research Laboratory, the U.S. Government, the U.K. Ministry of Defence or the U.K. Government. The U.S. and U.K. Governments are authorized to reproduce and distribute reprints for Government purposes notwithstanding any copyright notation hereon. 
Fabio Vitale acknowledges support from the ERC Starting Grant ``DMAP 680153'', the Google Focused Award ``ALL4AI'', and grant ``Dipartimenti di Eccellenza 2018-2022'', awarded to the Department of Computer Science of Sapienza University.

%% file: Appendix.tex
% !TEX root = OFLCA.tex

\clearpage

\newcommand{\concepts}{\mathcal{X}}
\newcommand{\hull}{\mathcal{H}}
\newcommand{\bs}[1]{\boldsymbol{#1}}

\appendix

\section{Missing Proofs}\label{MP}
{\bf Proof of Theorem \ref{Feas}}
\begin{proof}
Define $\fc{0}:=0$. We have $\spl{t}=\{\al{q}{k}:q\in\nem,k\leq\nm{q}+1\}\setminus\{0\}$ so:
\begin{align*}
\sum_{i\in\spl{t}}\fc{i}&\leq\sum_{q\in\nem,k\leq\nm{q}+1}\fc{\al{q}{k}}\\
&=\sum_{q\in\nem}\sum_{k\leq\nm{q}+1}\fc{\al{q}{k}}\\
&\leq\sum_{q\in\nem}\sum_{k\leq\nm{q}+1}\ca^{q-1}\ub\\
&=\sum_{q\in\nem}(\nm{q}+1)\ca^{q-1}\ub\\
&\leq\sum_{q\in\nem}(\cb\sz{q}+1)\ca^{q-1}\ub\\
&=\ub\sum_{q\in\nem}\ca^{q-1}+\ub\cb\sum_{q\in\nem}\sz{q}\ca^{q-1}\\
&\leq\ub\sum_{q=1}^{\infty}\ca^{q-1}+\ub\cb\sum_{q\in\nem}\sz{q}\ca^{q-1}\\
&=\frac{\ub}{1-\ca}+\ub\cb\sum_{q\in\nem}\sz{q}\ca^{q-1}\\
&=\frac{\ub}{1-\ca}+\ub\cb\sum_{q\in\nem}\sum_{i\in\pa{q}}\tw{t}{i}\ca^{q-1}\\
&=\frac{\ub}{1-\ca}+\cb\sum_{q\in\nem}\sum_{i\in\pa{q}}\tw{t}{i}\ca^{-1}(\ca^q\ub)\\
&\leq\frac{\ub}{1-\ca}+\cb\sum_{q\in\nem}\sum_{i\in\pa{q}}\tw{t}{i}\ca^{-1}\fc{i}\\
&\leq\frac{\ub}{1-\ca}+\frac{\cb}{\ca}\sum_{q\in\nem}\sum_{i\in\pa{q}}\tw{t}{i}\fc{i}\\
&\leq\frac{\ub}{1-\ca}+\frac{\cb}{\ca}\sum_{i\in[n]}\tw{t}{i}\fc{i}\\
&\leq\frac{\ub}{1-\ca}+\frac{\cb}{\ca}\\
&=\sqrt{\ub}+(1-\sqrt{\ub})\\
&=1
\end{align*}
\end{proof}

{\bf Proof of Lemma \ref{pl}}
\begin{proof}
Given some $\dd\in\nd$ let $\pe{\dd}$ be the event that $\ev$ happens on draw $\dd$. Given some set $\nd'\subseteq\nd$ let $\pf{\nd'}$ be the probability that $\ev$ happens on either of the draws in $\nd'$. We prove the result by induction on $|\nd|$. The inductive hypothesis clearly holds for $|\nd|=0$ as then $\pr{\pf{\nd}}=0=1-1=1-\exp(0)= 1-\exp\left(-\sum_{\dd\in\nd}\pp{\dd}\right)$.

Now suppose the inductive hypothesis holds for $|\nd|=l$ for some $l$. We now show that it holds for $|\nd|=l+1$ which will complete the proof. To show this choose some $\dd\in\nd$ and set $\nd'=\nd\setminus\{\dd\}$ We have:
\begin{align}
\nonumber\pr{\nt{\pf{\nd}}}&=\pr{\nt{\pf{\nd'}}\wedge\nt{\pe{\dd}}}\\
\nonumber&=\pr{\nt{\pf{\nd'}}}\pr{\nt{\pe{\dd}}}\\
\label{ll1}&=\pr{\nt{\pf{\nd'}}}(1-\pr{{\pe{\dd}}})\\
\nonumber&=\pr{\nt{\pf{\nd'}}}(1-\pp{\dd})\\
\nonumber&\leq\pr{\nt{\pf{\nd'}}}\exp(-\pp{\dd})\\
\nonumber&\leq(1-\pr{{\pf{\nd'}}})\exp(-\pp{\dd})\\
\label{ll2}&\leq\exp\left(-\sum_{\dd\in\nd'}\pp{\dd}\right)\exp(-\pp{\dd})\\
\nonumber&=\exp\left(-\sum_{\dd\in\nd}\pp{\dd}\right)
\end{align}
Where Equation \ref{ll1} is due to the independence of the draws and Equation \ref{ll2} is from the inductive hypothesis (noting $|\nd'|=l$). We now have
\begin{align}
\pr{{\pf{\nd}}}&=1-\pr{\nt{\pf{\nd}}}\\
&\geq1-\exp\left(-\sum_{\dd\in\nd}\pp{\dd}\right)
\end{align}
which proves the inductive hypothesis.
\end{proof}

{\bf Proof of Theorem \ref{pal1}}
\begin{proof}
Note that the event $\sus\cap\spl{t}\neq\emptyset$ happens if either of $\{\al{q}{k}:q\in\mathbb{N},k\leq\nm{q}\}$ are in $\sus$. For every $q\in\mathbb{N}$ and $k\in[\nm{q}]$ we have:
\begin{align*}
 \pr{\al{q}{k}\in\sus}&=\sum_{i\in\sus\cap\pa{q}}\pr{\al{q}{k}=i}\\
 &=\sum_{i\in\sus\cap\pa{q}}\frac{\we{t}{i}}{\sz{q}}
 \end{align*}
and for $k=\nm{q}+1$ we similarly have:
$$ \pr{\al{q}{k}\in\sus}=(\cb\sz{q}-\lfloor\cb\sz{q}\rfloor)\sum_{i\in\sus\cap\pa{q}}\frac{\we{t}{i}}{\sz{q}}$$
So letting $\pp{(q,k)}:= \pr{\al{q}{k}\in\sus}$ we have:
\begin{align*}
&\sum_{k=1}^{\nm{q}+1}\pp{(q,k)}\\
=&\sum_{k=1}^{\nm{q}}\sum_{i\in \sus\cap\pa{q}}\frac{\we{t}{i}}{\sz{q}}+(\cb\sz{q}-\lfloor\cb\sz{q}\rfloor)\sum_{i\in\sus\cap\pa{q}}\frac{\we{t}{i}}{\sz{q}}\\
=&\nm{q}\sum_{i\in \sus\cap\pa{q}}\frac{\we{t}{i}}{\sz{q}}+(\cb\sz{q}-\lfloor\cb\sz{q}\rfloor)\sum_{i\in\sus\cap\pa{q}}\frac{\we{t}{i}}{\sz{q}}\\
=&(\nm{q}+\cb\sz{q}-\lfloor\cb\sz{q}\rfloor)\sum_{i\in\sus\cap\pa{q}}\frac{\we{t}{i}}{\sz{q}}\\
=&\cb\sz{q}\sum_{i\in\sus\cap\pa{q}}\frac{\we{t}{i}}{\sz{q}}\\
=&\cb\sum_{i\in\sus\cap\pa{q}}{\we{t}{i}}
\end{align*} 
So, plugging into Lemma \ref{pl} with $\nd:=\{(q,k):q\in\mathbb{N},k\in[\nm{q}+1]\}$, we get:
\begin{align*}
\pr{\sus\cap\spl{t}\neq\emptyset}&=\pr{\exists i\in \sus: i\in\spl{t}}\\
&\geq1- \exp\left(-\sum_{\dd\in\nd}\pp{\dd}\right)\\
&= 1- \exp\left(-\sum_{q=1}^{\infty}\sum_{k=1}^{\nm{q}+1}\pp{(q,k)}\right)\\
&=1-\exp\left(-\cb\sum_{q=1}^{\infty}\sum_{i\in\sus\cap\pa{q}}{\we{t}{i}}\right)\\
&=1-\exp\left(-\cb\sum_{i\in\sus}\we{t}{i}\right)
\end{align*}
\end{proof}

{\bf Proof of Theorem \ref{ps}}
\begin{proof}
From Theorem \ref{pal1} we have $\pr{i\in\spl{t}}=\pr{\{i\}\cap\spl{t}\neq\emptyset}\geq1-\exp(-\cb\we{t}{i})$. Choosing $q\in\mathbb{N}$ such that $i\in\pa{q}$ we also have:
\begin{align*} 
\pr{i\in\spl{t}}&\leq\sum_{k=1}^{\nm{q}+1}\pr{\al{q}{k}=i}\\
&=\nm{q}\we{t}{i}/\sz{q}+(\cb\sz{q}-\lfloor\cb\sz{q}\rfloor)\we{t}{i}/\sz{q}\\
&=\cb\sz{q}\we{t}{i}/\sz{q}\\
&=\cb\we{t}{i}
\end{align*}
\end{proof}

{\bf Proof of Theorem \ref{l1}}
\begin{proof}
For all $j\in[n-1]$ the function $\cb\sum_{k=1}^j\wv{\s{t}{k}}$ is concave and the function $(1-\exp(-x))$ is concave and monotonic increasing which implies their combination, $(1-\exp\left(-\cb\sum_{k=1}^j\wv{\s{t}{k}}\right))$, is concave. Hence, as $\left(\rwd{t}{\s{t}{j}}-\rwd{t}{\s{t}{j+1}}\right)\geq0$, $\fun{t}{j}$ is concave. Similarly, as all components of $\bnc{t}$ are negative, we have that $\sum_{i=1}^n\cnc{t}{i}(1-\exp(-\cb\wv{i}))$ is convex.
Since also the function $\bpc{t}\cdot\gav$ is convex we then have that $\tf{t}$ is a positive sum of convex functions and is therefore convex.
\end{proof}

{\bf Proof of Lemma \ref{t1}}
\begin{proof}
Let $l:=\min\{j\in[n]:\s{t}{j}\in\spl{t}\}$. Note that $\rwd{t}{\s{t}{l}}=\max_{i\in\spl{t}}\rwd{t}{i}$.

From the definition of $l$ we have:
\begin{itemize}
\item For all $j<l$, $\id{\exists k\leq j: \s{t}{k}\in\spl{t}}=0$
\item For all $j\geq l$, $\id{\exists k\leq j: \s{t}{k}\in\spl{t}}=1$
 \end{itemize}
 This implies that:
 \begin{align*}
&\sum_{j=1}^n\left(\rwd{t}{\s{t}{j}}-\rwd{t}{\s{t}{j+1}}\right)\id{\exists k\leq j: \s{t}{k}\in\spl{t}}\\
=&\sum_{j=l}^n\left(\rwd{t}{\s{t}{j}}-\rwd{t}{\s{t}{j+1}}\right)\\
=&\rwd{t}{\s{t}{l}}\\
=&\max_{i\in\spl{t}}\rwd{t}{i}
 \end{align*}
\end{proof}

{\bf Proof of Theorem \ref{l4}}
\begin{proof}
For all $j\in[n]$, Theorem \ref{pal1} with $\sus:=\{\s{t}{k}:k\leq j\}$ implies that:
\begin{align*} 
&\pr{\exists k\leq j: \s{t}{k}\in\spl{t}}\\
=&\pr{\{\s{t}{k}:k\leq j\}\cap\spl{t}\neq\emptyset}\\
\geq&1-\exp\left(-\cb\sum_{k=1}^{j}\we{t}{\s{t}{k}}\right)
\end{align*} Lemma \ref{l2} then gives us:
\begin{align*}
&\ex{\max_{i\in\spl{t}}\rwd{t}{i}}\\
\geq&\sum_{j=1}^n\left(\rwd{t}{\s{t}{j}}-\rwd{t}{\s{t}{j+1}}\right)\left(1-\exp\left(-\cb\sum_{k=1}^{j}\we{t}{\s{t}{k}}\right)\right)\\
=&\sum_{j=1}^n\fun{t}{j}(\wev{t})
\end{align*}
By Theorem \ref{ps} we also have:
\begin{align*}
&\sum_{i\in[n]:\cst{t}{i}< 0}\cst{t}{i}~\pr{i\in\spl{t}}\\
\leq&\sum_{i\in[n]:\cst{t}{i}< 0}\cst{t}{i}(1-\exp(-\cb\we{t}{i}))\\
=&\sum_{i\in[n]}\cnc{t}{i}(1-\exp(-\cb\we{t}{i}))
\end{align*}
and: 
\begin{align*}
&\sum_{i\in[n]:\cst{t}{i}> 0}\cst{t}{i}~\pr{i\in\spl{t}}\\
\leq&\sum_{i\in[n]:\cst{t}{i}>0}\cst{t}{i}\cb\we{t}{i}\\
=&\sum_{i\in[n]}\cpc{t}{i}\cb\we{t}{i}\\
=&\cb\langle\bpc{t},\wev{t}\rangle
\end{align*}
so:
$$\sum_{i=1}^n\cst{t}{i}~\pr{i\in\spl{t}}\leq \cb\langle\bpc{t},\wev{t}\rangle +\sum_{i\in[n]}\cnc{t}{i}(1-\exp(-\cb\we{t}{i}))$$
Hence, we have:
\begin{align}
\nonumber&\ex{\lt{t}}\\
\nonumber=&\ex{\max_{i\in\spl{t}}\rwd{t}{i}-\sum_{i\in\spl{t}}\cst{t}{i}}\\
\nonumber=&\ex{\max_{i\in\spl{t}}\rwd{t}{i}-\sum_{i=1}^n\cst{t}{i}~\id{i\in\spl{t}}}\\
\label{l43}=&\ex{\max_{i\in\spl{t}}\rwd{t}{i}}-\sum_{i=1}^n\cst{t}{i}~\ex{\id{i\in\spl{t}}}\\
\nonumber=&\ex{\max_{i\in\spl{t}}\rwd{t}{i}}-\sum_{i=1}^n\cst{t}{i}~\pr{i\in\spl{t}}\\
\nonumber\geq&\sum_{j=1}^n\fun{t}{j}(\wev{t})-\sum_{i=1}^n\cst{t}{i}~\pr{i\in\spl{t}}\\
\nonumber\geq&\sum_{j=1}^n\fun{t}{j}(\wev{t})-\cb\langle\bpc{t},\wev{t}\rangle -\sum_{i\in[n]}\cnc{t}{i}(1-\exp(-\cb\we{t}{i}))\\
\nonumber=&-\tf{t}(\wev{t})
\end{align}
where Equation \ref{l43} is due to linearity of expectation.
\end{proof}

{\bf Proof of Theorem \ref{deriv}}
\begin{proof}
Suppose we have some $l\in[n]$. For $j<l$ we have $\pd{\s{t}{l}}{\fun{t}{j}}(\wev{t})=0$ and for $j\geq l$ we have 
\begin{align*}
&\pd{\s{t}{l}}{\fun{t}{j}}(\wev{t})\\
=&\cb\left(\rwd{t}{\s{t}{j}}-\rwd{t}{\s{t}{j+1}}\right)\exp\left(-\cb\sum_{k=1}^j\we{t}{\s{t}{k}}\right)\\=&\cb\left(\rwd{t}{\s{t}{j}}-\rwd{t}{\s{t}{j+1}}\right)\epsilon_j^t
\end{align*}
This implies that:
\begin{align*} \pd{\s{t}{l}}{}\sum_{j=1}^n\fun{t}{j}(\wev{t})&=\cb\sum_{j=l}^n\left(\rwd{t}{\s{t}{j}}-\rwd{t}{\s{t}{j+1}}\right)\epsilon_j^t\\
&=\cb\lambda^{t}_{{l}}
\end{align*}
which gives us
 \begin{align*}
& \pd{\s{t}{l}}{\tf{t}}(\wev{t})\\
 =&\cb\cpc{t}{\s{t}{j}}+\cb\cnc{t}{\s{t}{j}}\exp(-\cb\we{t}{\s{t}{j}})-\cb\lambda^{t}_{{l}}\\
 =&\gra{t}{\s{t}{l}}
 \end{align*}
\end{proof}

{\bf Proof of Theorem \ref{bound}}
\begin{proof}
From Lemma \ref{Gl} we have:
\begin{align*}
\sum_{t=1}^T\lr{t}\sqn{\gv{t}}^2&\leq \sum_{t=1}^T n\sqrt{s/(2t)}\\
&=n\sqrt{s/2}\sum_{t=1}^T\frac{1}{\sqrt{t}}\\
&\leq n\sqrt{s/2}\left(1+\int_{t=1}^T\frac{1}{\sqrt{t}}\right)\\
&= n\sqrt{s/2}(2\sqrt{T}-1)\\
&\leq 2n\sqrt{sT/2}
\end{align*}
So using Lemma \ref{Rl} and plugging into Theorem \ref{gdt} we have:
\begin{align*}
\sum_{t=1}^T(\tf{t}(\wev{t})-\tf{t}(\pv))&\leq\frac{1}{2}n\sqrt{2sT}+n\sqrt{sT/2}\\
&={2}n\sqrt{sT/2}\\
&=n\sqrt{2sT}
\end{align*}
 By Definition \ref{maxgrad} and Lemma \ref{grad} we have $s\leq\cb^2(\mr+\mc)^2$. Plugging this into the above gives us the result.
\end{proof}

\section{When $\max_{i\in [n]}\fc{i}$ is large}

When $\max_{i\in[n]}\fc{i}$ is large we construct $\spl{t}$ from $\we{t}{i}$ in the following way:
\begin{enumerate}
\item Define $A=\{i\in[n]:\fc{i}\in[1/2,1]\}$
\item Define $a^t:=\sum_{i\in A}\we{t}{i}/4$
\item Flip a biased coin with probability of heads equal to $a^t$
\item If heads sample $i$ from $A$ with probability $\we{t}{i}/(4a_t)$ and set $\spl{t}=\{i\}$
\item If tails define $\beta:=1/2$ and run Algorithm 1.
\end{enumerate}
We defer the analysis.

\section{On the Knapsack Problem and  Component Hedge}\label{knapsack}

In this section we show that Component Hedge cannot solve the online knapsack problem special case in polynomial time unless $P=NP$.

The Component Hedge algorithm described in \cite{Koolen10} requires a set of concepts $\concepts\subseteq \{0,1\}^n$. In the knapsack problem, $\concepts$ is the subset of all $\boldsymbol{x}$ that fit in the knapsack, i.e. those $\boldsymbol{x}\in\{0,1\}^d$ with $\langle\boldsymbol{x},\boldsymbol{z}\rangle\leq 1$. For Component Hedge, the convex hull $\hull$ of $\concepts$ must have a number of constraints polynomial in $n$ and that any $\boldsymbol{y}\in\hull$ can, in polynomial time, be decomposed into a convex combination of $n+1$ concepts. We now show that if this is true then we can solve the knapsack problem in polynomial time:

Let $\boldsymbol{c}$ be the vector defining the objective function of the knapsack problem (i.e. we seek the $\boldsymbol{x}\in \mathcal{X}$ that maximises $\langle\boldsymbol{c},\bs{x}\rangle$). Now, since $\hull$ has polynomial constraints we can efficiently choose the maximiser of $\langle\bs{c},\bs{y}\rangle$ (for $\bs{y}\in \hull$) via linear programming. Let this maximiser be $\bs{y}$. Now decompose $\bs{y}$ into a convex combination, $\sum_{i=1}^{n+1} (m_i \bs{x}^i)$, of $n+1$ concepts $\bs{x}^i$ (where $\sum_{i=1}^{n+1} m_i=1$). Now suppose, for contradiction, that $\langle\bs{c},\bs{x}^j\rangle<\langle\bs{c},\bs{y}\rangle$ for some $j$ with $m_j>0$. Without loss of generality let $j=n+1$.

Then, since for all $\bs{x}\in \concepts$ with have $\langle\bs{c},\bs{x}\rangle\leq \langle\bs{c},\bs{y}\rangle$ (as $\concepts\subset \hull$) we have:
$\langle\bs{c},\bs{y}\rangle=\langle\bs{c},(\sum_{i=1}^{n+1} m_i \bs{x}^i)\rangle=\sum_{i=1}^{n+1}m_i \langle\bs{c},\bs{x}^i\rangle=\sum_{i=1}^nm_i \langle\bs{c},\bs{x}\rangle^i+m_{n+1}\langle\bs{c},\bs{x}_{n+1}\rangle<\sum_{i=1}^nm_i \langle\bs{c},\bs{x}^i\rangle+m_{n+1}\langle\bs{c},\bs{y}\rangle\leq\sum_{i=1}^nm_i \langle\bs{c},\bs{y}\rangle+m_{n+1}\langle\bs{c},\bs{y}\rangle=\sum_{i=1}^{n+1}m_i \langle\bs{c},\bs{y}\rangle=\langle\bs{c},\bs{y}\rangle$
which is a contradiction.

Hence, for all $j$ with $m_j>0$ we have $\langle\bs{c},\bs{x}^j\rangle=\langle\bs{c}, \bs{y}\rangle$ so the maximising $\bs{x}$ is found in polynomial time. This contradicts the assumption that $P\neq NP$.

\section{Using Submodular Maximisation}\label{sm}

We consider the facility location special case (i.e, $\bv:=\boldsymbol{0}$ and all costs are positive). In this case it is possible to obtain our bound by selecting, on a trial $t$, each action $i$ with probability $\we{t}{i}$. We can then use some definitions in our paper to define the following functions on a trial~$t$:
\begin{itemize}
\item For all $j\in[n]$ define $\hat{f}^t_j(\gav):=$\\$\left(\rwd{t}{\s{t}{j}}-\rwd{t}{\s{t}{j+1}}\right)\left(1-\prod_{k=1}^j\left(1-\wv{\s{t}{k}}\right)\right)$
\item Define $\hat{h}^t(\gav)=\sum_{j=1}^n\hat{f}^t_j(\gav)-\langle\cv{t},\gav\rangle$
\end{itemize}
We can then use the arguments in our paper to write the expected profit as: $\ex{\lt{t}}=\hat{h}^t(\wev{t})$ which is continuous submodular. We can use a slight modification of our method of computing the gradient of $\tf{t}$ to compute the gradient of $\hat{h}^t$ in quasi-linear time. With this in hand we can then plug the gradient into one of the algorithms of \cite{Chen18}:
\begin{itemize}
\item The first algorithm requires the submodular function to be monotone and non-negative whereas ours, in general, is neither. Even if our expected profit was monotone and non-negative, for their algorithm to have a regret linear in $\sqrt{T}$ their per-trial time complexity bound becomes $\Omega(n\sqrt{T})$ which is much larger than ours.
\item In the gradient ascent based second algorithm, the submodular function also needs to be monotone and non-negative but the cost vector can be separated from the reward vector in the analysis allowing us to use the expected maximum reward which is monotone non-negative. The result of this analysis is similar to ours expect that for them $\alpha=1/2$ instead of the better $\alpha=1-e^{-1}$ that we have (NB: $\alpha$ is the discount on the comparator selection $\fp$).
\end{itemize}

%% file: OFLCA.bbl
\begin{thebibliography}{10}

\bibitem{Charikar99}
M. Charikar, S. Guha, E. Tardos, D.B. Shmoys. 
\newblock{A constant--factor approximation algorithm for the k--median problem}. 
In {\em ACM Symposium on Theory of Computing}
(STOC), pp. 1--10, ACM (1999).

\bibitem{Charikar99b}
M. Charikar and S. Guha. 
\newblock{Improved combinatorial
algorithms for the facility location and k--median
problems}. 
\newblock In {\em IEEE Foundations of Computer
Science}, pages 378--388, 1999.

\bibitem{Chen18}
L. Chen, H. Hassani, A. Karbasi.
\newblock{Online Continuous Submodular Maximization}. 
\newblock In {\em International Conference on Artificial Intelligence and Statistics}, AISTATS 2018. 


\bibitem{Cornuejols90}
G. Cornuejols, G. L. Nemhauser, and L. A. Wolsey. 
\newblock{The uncapacitated facility location problem}. 
\newblock In {\em Pitu B. Mirchandani and Richard L. Francis, editors, Discrete Location Theory}, pages 119--171. John Wiley and Son, Inc., New York, 1990.


\bibitem{Cygan18}
M. Cygan, A. Czumaj, M. Mucha, P. Sankowski.
\newblock {Online Facility Location with Deletions}.
\newblock In {\em Annual European Symposium on Algorithms}, ESA 2018.

\bibitem{Dantzig57}
 G.B. Dantzig.
\newblock {Discrete variable extremum problems},
\newblock In {\em Operations Research 5} (1957) 266--277.

\bibitem{Freund97}
Y. Freund and R. E. Schapire. 
\newblock {A decision-theoretic generalization of on-line learning and an application to boosting}. 
In {\em Journal of Computer and System Sciences}, 55:119--139, 1997.

\bibitem{Fujita13}
T. Fujita, K. Hatano, E Takimoto
\newblock {Combinatorial Online Prediction via Metarounding}. 
In {\em Algorithmic Learning Theory} (2013), 68-82,

\bibitem{Golovin14}
D. Golovin, A. Krause, M. Streeter.
\newblock {Online Submodular Maximization under a Matroid Constraint with Application to Learning Assignments}
In {\em Technical report, arXiv, 2014}.

\bibitem{Jain01}
K. Jain and V. Vazirani. 
\newblock{Approximation
algorithms for metric facility location and k--median problems using the primal--dual schema and Lagrangian relaxation}. 
\newblock In {\em J. ACM}, 48(2):274--296, 2001.

\bibitem{Kalai05}
 A. Kalai and S. Vempala.
 \newblock{Efficient algorithms for online decision problems},
\newblock In {\em Journal of Computer and System Sciences}, vol. 71, no. 3, pp. 291--307, 2005.

\bibitem{Kakade07}
 S. Kakade, A. Kalai, and K. Ligett. 
  \newblock{Playing games with approximation algorithms}.
\newblock In {\em ACM Symposium on the Theory of Computing}, pages 546-555, STOC 2007.

\bibitem{Koolen10}
W.M. Koolen, M.K. Warmuth, J. Kivinen.
\newblock{Hedging structured concepts}.
In {\em Conference on Learning Theory}, Omnipress,
2010, pp. 239-254.

\bibitem{Kiwiel08}
K. C. Kiwiel.
\newblock{Breakpoint searching algorithms for the continuous quadratic knapsack problem}. 
In {\em Math. Program}, 112(2): 473-491 (2008).

\bibitem{Kumar12}
A. Kumar. 
\newblock{Constant factor approximation algorithm for the knapsack median problem}.
In {\em ACM--SIAM Symposium on Discrete Algorithms} (SODA),
pp. 824--832, SIAM (2012).

\bibitem{Hazan12}
E. Hazan, and S. Kale. 
\newblock{Online submodular minimization}.
In {\em Journal of Machine Learning Research
(JMLR)}, 2012.

\bibitem{Helmbold18}
D. P. Helmbold and S. V. N. Vishwanathan.
\newblock{Online Learning of Combinatorial Objects via Extended Formulation}. 
In {\em International Conference on Algorithmic Learning Theory}, ALT 2018. 

\bibitem{Laoutaris07}
N. Laoutaris, G. Smaragdakis, K. Oikonomou, I. Stavrakakis, A. Bestavros.
\newblock {Distributed Placement of Service Facilities in Large--Scale Networks}.
\newblock In {\em IEEE INFOCOM}, pages 2144--2152, 2007.

\bibitem{Martello77}
S. Martello, P. Toth. 
\newblock {An upper bound for the zero--one
knapsack problem and a branch and bound algorithm},
\newblock In {\em European Journal of Operational Research 1} (1977) 169--175.

\bibitem{Meyerson01}
A. Meyerson.
\newblock {Online Facility Location}.
\newblock In {\em IEEE Symposium on Foundations of Computer Science}, FOCS 2001.

\bibitem{Shmoys01}
D. Shmoys, E. Tardos and K. Aardal. 
\newblock {Approximation algorithms for facility location problems}. 
\newblock In {\em ACM Symposium on Theory of Computing} (STOC 1997), pages 265--274, ACM Press, 1997.

\bibitem{Streeter08}
M. Streeter and D. Golovin.
\newblock {An Online Algorithm for Maximizing Submodular Functions}.
\newblock In {\em Conference on Neural Information Processing Systems}, NIPS 2008.


\bibitem{Yu17}
H. Yu, M. J. Neely, X. Wei.
\newblock {Online Convex Optimization with Stochastic Constraints}.
\newblock In {\em Conference on Neural Information Processing Systems}, NIPS 2017.
 
\bibitem{Zinkevich03}
M. Zinkevich. 
\newblock {Online convex programming and generalized infinitesimal gradient ascent}. 
\newblock In {\em International Conference on Machine Learning}, 2003.

\bibitem{ServicePlacement}
S. Pasteris, S. Wang, M. Herbster, T. He.
\newblock{Service Placement with Provable Guarantees in Heterogeneous Edge Computing Systems}.
\newblock To appear in {\em IEEE INFOCOM}, 2019.


\end{thebibliography}
